\documentclass{article}

\usepackage[preprint]{neurips_2023}

\usepackage[utf8]{inputenc} 
\usepackage[T1]{fontenc}    
\usepackage{hyperref}       
\usepackage{url}            
\usepackage{booktabs}       
\usepackage{amsfonts}       
\usepackage{nicefrac}       
\usepackage{microtype}      
\usepackage{xcolor}         
\usepackage{multirow}

\usepackage{amsmath}
\usepackage{amssymb}
\usepackage{mathtools}
\usepackage{amsthm}

\usepackage{algorithm}
\usepackage[noend]{algpseudocode}
\usepackage{comment}
\renewcommand{\algorithmicrequire}{\textbf{Input:}}
\renewcommand{\algorithmicensure}{\textbf{Output:}}

\usepackage{microtype}
\usepackage{graphicx}
\usepackage{subfig}
\usepackage{booktabs} 
\usepackage{thmtools}
\usepackage{thm-restate}
\usepackage{wrapfig}
\usepackage{lipsum}
\usepackage{array}
\usepackage{bbm}
\newcolumntype{?}{!{\vrule width 1pt}}
\newcommand{\ours}{PDP-SGD}

\title{Personalized DP-SGD using Sampling Mechanisms}

\author{%
  Geon Heo\\
  KAIST\\
  \texttt{geon.heo@kaist.ac.kr} \\
  \And
  Junseok Seo \\
  KAIST \\
  \texttt{ap9598@kaist.ac.kr} \\
  \AND
  Steven Euijong Whang\thanks{Corresponding author} \\
  KAIST\\
  \texttt{swhang@kaist.ac.kr} \\
}

\begin{document}

\theoremstyle{plain}
\newtheorem{theorem}{Theorem}[section]
\newtheorem{proposition}[theorem]{Proposition}
\newtheorem{lemma}[theorem]{Lemma}
\newtheorem{corollary}[theorem]{Corollary}
\theoremstyle{definition}
\newtheorem{definition}[theorem]{Definition}
\newtheorem{assumption}[theorem]{Assumption}
\theoremstyle{remark}
\newtheorem{remark}[theorem]{Remark}

\newcommand{\swhang}[1]{\textcolor{blue}{#1}}
\newcommand{\geon}[1]{\textcolor{purple}{#1}}

\maketitle

\begin{abstract}

Personalized privacy becomes critical in deep learning for Trustworthy AI. 
While Differentially Private Stochastic Gradient Descent (DP-SGD) is widely used in deep learning methods supporting privacy, it provides the same level of privacy to all individuals, which may lead to overprotection and low utility.
In practice, different users may require different privacy levels, and the model can be improved by using more information about the users with lower privacy requirements. 
There are also recent works on differential privacy of individuals when using DP-SGD, but they are mostly about individual privacy accounting and do not focus on satisfying different privacy levels. 
We thus extend DP-SGD to support a recent privacy notion called $(\Phi,\Delta)$-Personalized Differential Privacy ($(\Phi,\Delta)$-PDP), which extends an existing PDP concept called $\Phi$-PDP. Our algorithm uses a multi-round personalized sampling mechanism and embeds it within the DP-SGD iterations.  
Experiments on real datasets show that our algorithm outperforms DP-SGD and simple combinations of DP-SGD with existing PDP mechanisms in terms of model performance and efficiency due to its embedded sampling mechanism.
\end{abstract}


\section{Introduction}

Protecting an individual's privacy against various recent threats in the machine learning (ML) process is considered one of the most critical goals of Trustworthy AI research. One of the main streams of privacy research is Differential Privacy (DP), which is a mathematical definition for quantifying and comparing data privacy against an arbitrary mechanism.

DP makes it possible to quantify and mathematically prove the risk of personal information leakage and the level of privacy protection on the system using the data. Although DP originates from the database field, it is now one of the most actively studied topics for personal information protection in machine learning. Recently, the U.S. Census uses DP to protect personal information~\citep{Garfinkel2022Differential}. In addition, libraries such as Facebook's Opacus~\citep{opacus} and Google's Tensorflow Privacy~\citep{DBLP:journals/corr/abs-1812-06210} provide DP methods that can be used in the training process of machine learning models.

The biggest drawback of existing DP methods is a trade-off with model performance. If the level of privacy protection is increased excessively, the data will be safe from various kinds of privacy attacks, but the performance of the model would drastically decrease as well. The conventional DP concept probabilistically guarantees protecting the entire data, but individual protection levels cannot be adjusted. For example, suppose there are three users and their privacy preferences: ($u_1$, $\epsilon_1=0.5$), ($u_2$, $\epsilon_2=0.9$), and ($u_3$, $\epsilon_3=0.1$) as illustrated in Figure~\ref{fig:pdp}a. In this case, the conventional DP mechanisms have no choice, but to apply the most strict standards ($\epsilon_3=0.1$) to protect the privacy of all users. Although $u_1$ and $u_2$ are relatively more generous for their privacy protection, their information is overprotected, which leads to model performance degradation.

\begin{figure}[t]
\centering
    \subfloat[]{
        {\includegraphics[scale=0.34]{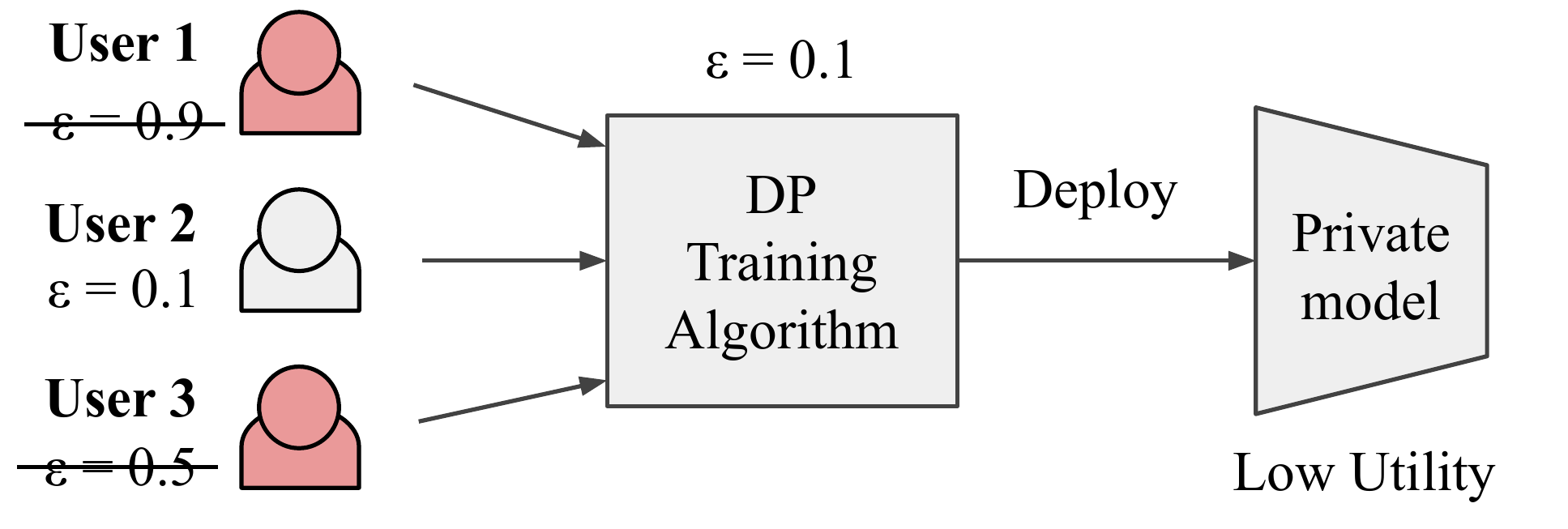}}
        }\hfill
    \subfloat[]{
        {\includegraphics[scale=0.34]{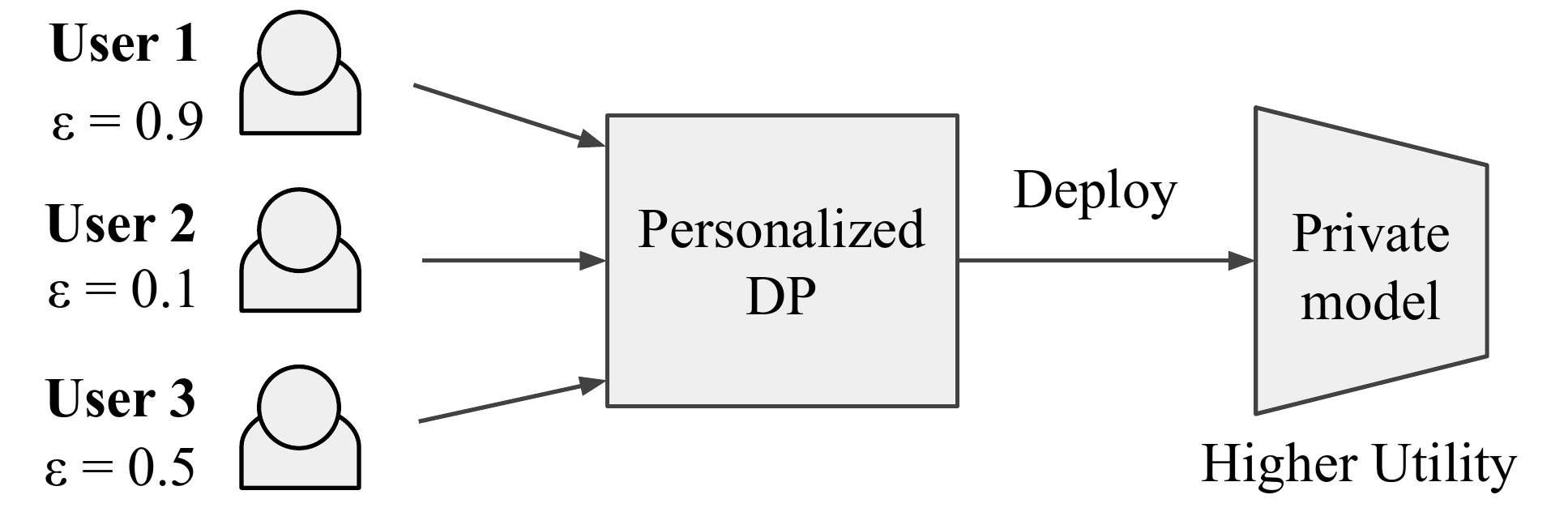}}
        }
    \caption{(a) Conventional DP overprotects by applying the strictest privacy. (b) Personalized Differential Privacy (PDP) avoids overprotection by allowing different $\epsilon$ values.}
\label{fig:pdp}
\vspace{-0.5cm}
\end{figure}


To compensate for this drawback, Personalized Differential Privacy (PDP)~\citep{jorgensen2015conservative} was proposed where one can specify different $\epsilon$ values for individuals in a privacy requirement set $\Phi$ and thus avoid overprotection (Figure~\ref{fig:pdp}b). As individual privacy in ML has recently gained interest, recent works on DP ML models, especially DP-SGD~\citep{abadi2016deep}, focus on quantifying individual accountability of privacy~\citep{feldman2021individual,da2022individual,koskela2022individual}. However, they only assume the setting with a uniform privacy budget. The closest work is \citet{muhl2022personalized}, which transforms PATE~\citep{papernot2016semi} with different privacy levels. However, it has disadvantages including high memory and computational costs.



In this work, we introduce how to extend $(\epsilon,\delta)$-DP mechanisms to the $(\Phi,\Delta)$-PDP (Definition~\ref{def:pdpdp}) using the personalized sampling mechanism. Here $\Phi$ reflects the individual privacy levels, and $\Delta$ reflects the failure probabilities of protection. We propose an ML training algorithm called \ours{} based on DP-SGD embedding the personalized sampling within the internal process. In addition, we show that \ours{} theoretically satisfies $(\Phi,\Delta)$-PDP. 


We summarize our contributions: (1) We propose a DP-SGD variant method internally embedding the personalized sampling mechanism with different privacy levels. (2) We provide its theoretical guarantee about $(\Phi,\Delta)$-PDP extending the conventional PDP concept. (3) We empirically show that \ours{} outperforms other conventional DP-SGD and simple combinations with PDP mechanisms in terms of performance and efficiency.


\section{Preliminaries}


\subsection{Differential Privacy \& DP-SGD}
\label{sec:dpsgd}

We first introduce the notion of $(\epsilon,\delta)$-differential privacy.

\begin{definition}
\label{def:edp}
($(\epsilon,\delta)$-Differential Privacy~\cite{dwork2006our})
A randomized mechanism $\mathcal{M}:\mathcal{D}\rightarrow\mathcal{R}$ with domain $\mathcal{D}$ and range $\mathcal{R}$ satisfies $(\epsilon,\delta)$-differential privacy if for any two adjacent inputs $d,d'\in\mathcal{D}$ such that $\left\| d-d' \right\|_1\leq 1$ and for any subset of outputs $O\subseteq R$ it holds that
\begin{equation*}
    Pr[\mathcal{M}(d) \in O]\leq e^{\epsilon}Pr[\mathcal{M}(d') \in O]+\delta.
\end{equation*}
\end{definition}

This is initially defined to bound privacy guarantees of the Gaussian mechanism that is defined as follows:
\begin{equation*}
    \mathcal{M}(d) \triangleq f(d)+\mathcal{N}(0,S^2_f \cdot \sigma^2)
\end{equation*}
where $S_f=\max_{d,d'}\left \| f(d)-f(d') \right \|_2$ for all adjacent inputs $d,d'$. The mechanism $\mathcal{M}$ satisfies $(\epsilon,\delta)$-DP if $\delta > \frac{4}{5} exp(-\frac{\sigma^2\epsilon^2}{2})$ and $\epsilon<1$~\cite{dwork2014algorithmic}.

Differentially Private Stochastic Gradient Descent (DP-SGD)~\citep{abadi2016deep} is one of the most widely used DP guaranteed deep learning algorithms, which adds noise to the gradient value during the training phase of deep learning. The algorithm satisfies ($\epsilon, \delta$)-differential privacy by adding Gaussian noise to the gradient of the model parameter while performing stochastic gradient descent. 

\subsection{Personalized DP \& Individualized DP}
\label{sec:pdp}

Conventional DP does not consider individual privacy preferences of users. Since all users are subject to the same privacy standard, there is a missed opportunity to improve utility by lowering the privacy of users who have more relaxed preferences. Personalized Differential Privacy (PDP) addresses this need by satisfying each user's privacy preference as follows:

\begin{definition}
\label{def:ppdp}
($\Phi$-Personalized Differential Privacy~\citep{jorgensen2015conservative})
In the context of a privacy specification $\Phi$ and a universe of users $\mathcal{U}$, a randomized mechanism $\mathcal{M}:\mathcal{D}\to R$ satisfies $\Phi$-personalized differential privacy ($\Phi$-PDP), if for every pair of neighboring datasets $D,D'\subset \mathcal{D}$, with $D\stackrel a\sim D'$, and for all sets $O\subseteq R$ of possible outputs,
\begin{equation}
    Pr[\mathcal{M}(D) \in O]\leq e^{\epsilon_i}\times Pr[\mathcal{M}(D') \in O],
\end{equation}
where $u_i\in \mathcal{U}$ is the user associated with tuple $d$, and $\epsilon_i \in \Phi$ denotes $u_i$'s privacy requirement.
\end{definition}

There are various mechanisms~\citep{alaggan2015heterogeneous,ebadi2015differential} that satisfy $\Phi$-PDP, and the most representative ones are the {\em sampling}~\citep{jorgensen2015conservative} and {\em partitioning}~\citep{li2017partitioning} mechanisms. Sampling mechanisms choose the data with different sampling probabilities based on their privacy preferences before applying the DP mechanism, while partitioning mechanisms split the data and perturb the results of each partition with possibly different privacy preferences. Among the two mechanisms, we use sampling as it is possible to extend it to satisfy our privacy notion of interest ($\Phi$, $\Delta$)-PDP, which is defined as follows:
\begin{definition}
\label{def:pdpdp}
(($\Phi$, $\Delta$)-Personalized DP)
In the context of a privacy specification $\Phi$ and a universe of users $\mathcal{U}$, a randomized mechanism $\mathcal{M}:\mathcal{D}\to R$ satisfies $(\Phi,\Delta)$-personalized differential privacy ($(\Phi,\Delta)$-PDP), if for every pair of neighboring datasets $D,D'\subset \mathcal{D}$ with $D\stackrel d\sim D'$, and for all sets $O\subseteq R$ of possible outputs,
\begin{equation}
    Pr[\mathcal{M}(D) \in O]\leq e^{\epsilon_i}\times Pr[\mathcal{M}(D') \in O]+\delta_i,
\end{equation}
where $u_i\in \mathcal{U}$ is a user associated with tuple $d$, $\epsilon_i \in \Phi$ denotes $u_i$'s privacy requirement, and $\delta_i \in \Delta$ is $u$'s probability for information leakage. 
\end{definition}
Definition ~\ref{def:pdpdp} is a generalization of $\Phi$-PDP and $(\epsilon,\delta)$-DP. If $\delta_i = 0 $ for all $u \in \mathcal{U}$, the definition becomes $\Phi$-PDP. If $\epsilon_i = \epsilon$ and $\delta_i = \delta$ for all $u_i \in \mathcal{U}$, the definition becomes $(\epsilon, \delta)$-DP. 



\citet{feldman2021individual}, \citet{da2022individual}, \citet{koskela2022individual}, and \citet{muhl2022personalized} use a similar definition called individualized $(\epsilon_d, \delta_d)$-DP, but this does not include user privacy requirements as a set. 

\begin{definition}
\label{def:idp}
(($\epsilon_d$, $\delta_d$)-Individualized DP)
For any data point $d \in \mathcal{D}$, a randomized mechanism $\mathcal{M}:\mathcal{D}\to R$ satisfies $(\epsilon_d,\delta_d)$-individualized differential privacy ($(\epsilon_d,\delta_d)$-IDP) if for every pair of neighboring datasets $D,D'\subset \mathcal{D}$ with $D\stackrel d\sim D'$, and for all sets $O\subseteq R$ of possible outputs,
\begin{equation}
    Pr[\mathcal{M}(D) \in O]\leq e^{\epsilon_d}\times Pr[\mathcal{M}(D') \in O]+\delta_d.
\end{equation}
\vspace{-0.7cm}
\end{definition}
The two concepts are near identical, but we use Personalized DP ($(\Phi,\Delta)$-PDP) in order to capture the scenario where users have different privacy requirements captured as a requirement set. In comparison, Individualized DP ($(\epsilon_d,\delta_d)$-IDP) focuses on satisfying the privacy of one user without regarding other users. We elaborate which concept has been used in which works in the next section.

\section{Related Work}
\label{sec:deeplearningwithindividualprivacy}

Individual Privacy Accounting and Individualized PATE have motivations similar to our work, and both assume deep learning situations. \citet{feldman2021individual} and \citet{koskela2022individual} mainly introduce a new adaptive composition theorem for RDP and GDP, respectively, and suggest a method estimating individual privacy costs. However, these techniques assume that the target privacy budget is the same for all users.
There are more individual privacy accounting methods for deep learning algorithms with DP -- Differentially Private Bagging~\citep{DBLP:conf/nips/JordonYS19} and Individual Privacy Accounting for DP-SGD~\citep{da2022individual} -- but they also only focus on how to check individual privacy loss accurately with a uniform privacy budget, whereas our work focuses more on improving utility when the privacy budgets are not uniform over the training set.

Individualized PATE~\citep{muhl2022personalized} transforms PATE~\citep{papernot2016semi} in two ways so that each data point can receive different privacy protections satisfying individualized DP. This work is the first to specify and implement individual privacy requirements of data holders in a deep learning setting. However, extending PATE also has disadvantages including high memory and computational costs. In one of its experiments, Individualized PATE requires 250 teacher networks to train a student model on the MNIST and SVHN datasets. In comparison, our methods does not need to train additional teacher models. A recent method called IDP-SGD\,\cite{boenisch2023way} also uses sampling to personalize DP-SGD. However, the sampling algorithm is stated at a high level and does not provide full theoretical details on how $(\Phi,\Delta)$-PDP is satisfied as in our work\footnote{Our research was conducted independently of IDP-SGD.}.



\section{Personalized DP-SGD}
\label{sec:pdpsgd}

We propose \ours{}, an algorithm based on DP-SGD that is tightly integrated with a personalized sampling mechanism. We first introduce the sampling mechanism and then explain how it can be embedded in DP-SGD (Section~\ref{sec:alg}). We also prove that \ours{} satisfies $(\Phi,\Delta)$-PDP (Section~\ref{sec:prove}).


\subsection{Algorithm}
\label{sec:alg}

\paragraph{Sampling Mechanism}
The {\em sampling mechanism} was first presented in \citet{jorgensen2015conservative} and is applicable to any $\epsilon$-DP algorithm. However, most DP deep learning algorithms often use Gaussian noise, which is difficult to interpret with $\epsilon$-DP. We will thus examine whether the sampling mechanism is applicable to the $(\epsilon, \delta)$-DP mechanism, and whether it satisfies  $(\Phi,\Delta)$-PDP. We first adapt the sampling mechanism to work with the $(\epsilon, \delta)$-DP mechanism as follows.

\begin{definition}
\label{def:pdpsam}
($(\epsilon,\delta)$-DP Mechanism with Sampling) Let $A_{DP}^{\tau,\delta}$ denote a randomized algorithm satisfying ($\tau, \delta$)-DP. For a dataset $D\subset \mathcal{D}$, let $f_{s}(D,\Phi,\tau)$ denote the preprocessing step that independently samples each data point $x_i \in D$ with probability
\begin{equation}
\label{eq:sam}
    \pi_i=\left\{
    \begin{matrix}
\frac{e^{\epsilon_i}-1}{e^{\tau}-1} \;\;\;\; \text{if}\;\; \epsilon_i<\tau\\ 
\;\;\;1\;\;\;\;\;\text{otherwise}
\end{matrix}\right.
\end{equation}
where $\min \Phi \leq \tau \leq \max \Phi $ is a configurable threshold with privacy budgets $\Phi=\{\epsilon_1,...,\epsilon_n\}$. Finally, the mechanism is defined as
\begin{equation}
    \mathcal{M}_{s}(D,\Phi,\tau) = A_{DP}^{\tau,\delta}(f_{s}(D,\Phi,\tau)).
\end{equation}
\end{definition}

The sampling mechanism is based on a technique called privacy amplification by sampling~\citep{DBLP:conf/nips/BalleBG18}. This mechanism protects the data points with different privacy standards by using different sampling probabilities for selecting data points. The mechanism takes a threshold and samples the data with a probability $\frac{e^{\epsilon_i}-1}{e^{\tau}-1}$ if $\epsilon_i$ is lower than the threshold. If $\epsilon_i$ is greater than the threshold, it is unconditionally sampled. This algorithm already exists for $\Phi$-PDP, but it has not been proved whether the algorithm can also be applied to an $(\epsilon,\delta)$-DP algorithm and satisfy $(\Phi,\Delta)$-PDP. We thus prove the following theorem:

\begin{restatable}[]{theorem}{samthm}
\label{thm:sam}
The mechanism $\mathcal{M}_{s}$ satisfies ($\Phi$,$\Delta$)-PDP, where $\Delta = \{\delta_i | \delta_i = \pi_i \delta,u_i\in \mathcal{U} \}$.
\end{restatable}

The proof can be found in Section~\ref{appendix:samtheoremproof}.

\paragraph{Thresholding}

The sampling mechanism improves the conventional DP mechanism in terms of reducing privacy waste, but can be improved by choosing better threshold $\tau$. Suppose we use the sampling mechanism to sample data with probability $\pi$ and use the sampled data for the ($\tau$, $\delta$)-DP mechanism. First, there is privacy waste in the sampled data due to the thresholding.  If $\epsilon_i=0.9$ for user $u_i$, and the threshold $\tau=0.6$, $u_i$ can be overprotected by a budget of 0.3. Second, there is also privacy waste for the rest of the data that is not sampled as it is completely unused (i.e., $\epsilon=0$), which results in a utility loss by overprotection. To further reduce the budget wastes, \citet{niu2021adapdp} suggests a solution to compute the threshold using the utility loss and apply a multiple sampling mechanism. So we adapt the thresholding method using the below loss.

The threshold is set in the range of $\min \Phi \leq \tau \leq \max \Phi$, and among these values, the value with the lowest estimated utility loss is selected. We estimate the loss using the following equation, which is a weighted sum of two losses:
\begin{equation}
\vspace{-0.2cm}
\label{eqn:floss}
    \begin{aligned}
        loss_{f}(\Phi,\,\tau)=w_1\cdot\underbrace{\sum_{i:\epsilon_i<t,\epsilon_i \in \Phi} \epsilon_i(1-\frac{e^{\epsilon_i}-1}{e^{\tau}-1})}_{waste_u(\Phi,\tau)}+w_2\cdot\underbrace{\sum_{i:\epsilon_i>t,\epsilon_i \in \Phi}(\epsilon_i-\tau)}_{waste_s(\Phi,\tau)}
    \end{aligned}
\end{equation}
where the weights $w_i$ are parameters for controlling the importance of each loss. 

The losses penalize the budget wastes of the two cases mentioned in the previous paragraph. $waste_u$ is when the data point is not sampled, and in this case, it is set to be proportional to $1-\pi_i$.  $waste_s$ is the waste caused by the threshold after sampling and is computed as the sum of $\epsilon_i-\tau$. In addition, we suggest an adaptive loss varying weights during multi-round sampling based on the inequality of arithmetic and geometric mean's equality condition between two type of wastes. The loss and weights are as follows:
\begin{equation}
\label{eqn:aloss}
    \begin{aligned}
        loss_a(\Phi,\tau)= \frac{2\cdot waste_u \cdot waste_s}{waste_u+waste_s} \; \mathrm{when} \;\;w_1=\frac{waste_u}{waste_u+waste_s},\; w_2=\frac{waste_s}{waste_u+waste_s} 
    \end{aligned}
\end{equation}
After deriving the minimum threshold $\tau$, we use it to run the sampling mechanism. In fact, this $loss_a$ does not return the exact optimal point but approximate minimum because the inequality of arithmetic and geometric mean's lower bound $2 \sqrt{w_1 \cdot w_2 \cdot waste_u \cdot waste_s}$ is not a constant. However, this $loss_a$ saves the trouble of finding optimal value for $w_1$, $w_2$ of $loss_f$. The comparison of accuracy results using $loss_f$ and $loss_a$ can be found in Section ~\ref{sec:acc}.

\paragraph{Embedding to DP-SGD}

We now describe the final \ours{} algorithm, which tightly integrates personalized sampling by embedding it within the SGD process (Algorithm~\ref{alg:emb}). We define a round to be a group of iterations. We set the number of iterations in one round to be $n$, and the total number of founds to be $k$. $\epsilon_R$ is a target privacy budget of the round $R$. For each round, we first compute a threshold (Step 2) for privacy budgets at the beginning of each round using the {\em fixed-weight} loss (Eqn.~\ref{eqn:floss}) or {\em adaptive-weight} loss (Eqn.~\ref{eqn:aloss}). We then invoke $getNoiseMultiplier$, which computes a noise multiplier for the given inputs (Step 3). In our setting, we use the noise scale computing function based on binary search in Opacus~\citep{opacus} (see detailed algorithm in Section~\ref{appendix:getnoisemltplr}). We then perform sampling using the rates in Definition~\ref{def:pdpsam} where higher privacy budgets lead to more frequent sampling (Step 4). The weighted sampler for PDP uses approximate Poisson sampling in Opacus~\citep{opacus} for efficiency (see more details in Section~\ref{appendix:pdpsampler}). We then apply DP-SGD for one round, which includes gradient clipping, adding noise, and gradient descent (Steps 6--9). We then invoke $PrivacyAccounting$, which can be any privacy accountant including RDP~\citep{mironov2017renyi}, GDP~\citep{bu2020deep}, or PRV~\citep{gopi2021numerical}, to produce $\epsilon'$ (Step 10). For our implementation, we use an RDP accountant. We then subtract the privacy budget of each sample by $\epsilon'$ (Step 11). Other implementation details are in Section~\ref{appendix:implementation}.

\begin{algorithm}[t]
\caption{Personalized DP-SGD ({\em PDP-SGD})}\label{alg:emb}
\begin{algorithmic}[1]
\Require Examples $D=\{x_1,...,x_n\}$, loss function $l(\theta)=\frac{1}{n}\Sigma_i l(\theta, x_i)$. 
Parameters: learning rate $\mu_R$, gradient norm bound $C$, number of rounds $k$, epochs per round $n$, privacy budgets $\Phi=\{\epsilon_1,...,\epsilon_n\}$, probability of privacy leakage $\delta$.

\For{$R=1,2,...,k$}

\State $\tau_R \gets \min(loss(\Phi,\,\tau)$,  $\min \Phi \leq \tau \leq \max \Phi$) \Comment{Compute threshold}
\State $\sigma \gets$ getNoiseMultiplier$(\tau_R , \delta, n)$

\State $D_R \gets$ Sampling($D$, $\Phi$, $\tau_R$) \Comment{Sampling Mechanism}

\For{$t=1,2,...,n$}

\State $g_R(x_i) \gets \nabla_\theta l(\theta_R ; x_i), \forall x_i \in D_R$ \Comment{Compute gradients}
\State $\bar{g_R}(x_i) \gets g_R(x_i)\cdot \min (1,\frac{C}{\left\|g_R(x_i)\right\|_2}), \forall x_i \in D_R$ \Comment{Clip gradients}
\State $\tilde{g_R} \gets \frac{1}{n} \sum_{i=1}^{n}(\bar{g_R}(x_i)+N(0,\sigma^2C^2\mathbb{I})), \forall x_i \in D_R$ \Comment{Add noise}
\State $\theta_{R+1} \gets \theta_R-\mu_R\tilde{g_R}$ \Comment{Take gradient step}

\EndFor

\State $\epsilon' \gets$ PrivacyAccounting$(\sigma,\,k,\,R)$

\State $\epsilon_i \gets \max(\epsilon_i-\epsilon', \,0)$, $\forall i \in \{i|x_i \in D_t\}$ \Comment{Compute remaining budgets}
\EndFor

\Ensure $\theta_R$ 

\end{algorithmic}

\end{algorithm}

\subsection{$(\Phi,\Delta)$-PDP Guarantee}
\label{sec:prove}

We show that our algorithm satisfies $(\Phi,\Delta)$-PDP. Before proving the privacy guarantee, we need to prove composition and post-processing theorems for $(\Phi, \Delta)$-PDP below. We note that these theorems have not been proved in any other work. From a privacy analysis standpoint, our algorithm can be viewed as a sequential aggregation of multiple sampling algorithms, each of which satisfies $(\Phi, \Delta)$-PDP according to Theorem~\ref{thm:sam}. 



\begin{restatable}[Composition Theorem]{theorem}{compositionthm}
\label{thm:cmp}
Let $T_1 (D) : D \rightarrow T_1 (D) \in \mathcal{C }_1$ be an $(\Phi_1 , \Delta_1 )$-PDP function, and for any $s_1 \in \mathcal{C}_1$, $T_2(D,s_1 ) : (D, s_1 ) \rightarrow T_2 (D,s_1 ) \in \mathcal{C}_2 $ be an $(\Phi_2 , \Delta_2 )$-PDP function given the second input $s_1$. We show that for any neighboring $D, D'$ such that $D\stackrel d\sim D'$ where $d$ is the data record of an arbitrary user $u_i \in \mathcal{U}$, for any $S \subseteq \mathcal{C}_2 \times \mathcal{C}_1$, 
\begin{align*}
    &Pr[(T_2 (D,T_1 (D)),T_1 (D) ) \in S] \leq\\
    &e^{\epsilon_{1i} + \epsilon_{2i}} Pr[(T_2 (D',T_1 (D')) , T_1 (D') ) \in S] + \delta_{1i} + \delta_{2i} 
\end{align*}
where $\epsilon_{1i} \in \Phi_1$, $\epsilon_{2i} \in \Phi_2$, $\delta_{1i} \in \Delta_1$, and $\delta_{2i} \in \Delta_2$.
\end{restatable}
This theorem is a PDP version of the composition theorem in ~\citet{dwork2013algorithmic}. The proof can be found in Section~\ref{appendix:compositiontheoremproof}.

\begin{restatable}[Post-processing for PDP]{theorem}{postpthm}
\label{thm:post}
Let $\mathcal{M}$ be a $(\Phi, \Delta)$-PDP algorithm. For an arbitrary randomized mapping $f$, $f \circ M$ is a $(\Phi, \Delta)$-PDP algorithm.
\end{restatable}

The proof can be found in Section~\ref{appendix:postptheoremproof}.

\begin{theorem}
{\em PDP-SGD} satisfies $(\Phi,\Delta)$-PDP, where $\Phi=\{\epsilon_1,...,\epsilon_n\}$ and $\Delta=\{\delta_i | \delta_i = R_{max}\delta, i\in \{1,...,n\}\}$.
\end{theorem}

\begin{proof}
Let $\mathcal{U} = \{ u_1 , ..., u_n \} $ be the set of users where user $u_i$ has privacy preference of $(\epsilon_i,\delta_i)$. Because DP-SGD is a $(\epsilon, \delta)$-DP mechanism, we can say that one round of {\em PDP-SGD} as a mechanism $\mathcal{M}_R$ that satisfies ~\ref{def:pdpsam}. Suppose that $\mathcal{M}_R$ uses $(\tau_R, \delta)$-DP mechanism. 

Let $\mathcal{U}_{s}$ be the set of users of sampled data in the round $R$. For any $u_s \in \mathcal{U}_s$ and $u_t \in \mathcal{U}-\mathcal{U}_s$, let $\epsilon_{s,R}$, $\epsilon_{t,R}$ be the spent privacy budget for the data of $u_s$, $u_t$ in the round $R$. As shown in \citep{niu2021adapdp}, $\epsilon_{t,R}$ is 0 because $\mathcal{M}_R$ does not use data for $u_t$, and $\epsilon_{s,R} = \tau_R$ for $\epsilon_s > \tau_R$, $\epsilon_{s,R} = \epsilon_s$ for $\epsilon_s <= \tau_R$. So we can say that $\mathcal{M}_R$ satisfies $(\Phi_R, \Delta_R)$-PDP where $\Phi_R = \{\epsilon_{s,R} | u_s \in \mathcal{U}_s \} \cup \{ \epsilon_{t,R} | u_t \in \mathcal{U}_t \} $ and $\Delta_R = \{\delta_i | \delta_i = \pi_{i,R} \delta, u_i \in \mathcal{U} \}$. 

By Theorem ~\ref{thm:cmp}, we can say that {\em PDP-SGD} satisfies $(\Phi',\Delta')$-PDP where $\Phi'= \{\epsilon_i | \epsilon_i = \sum_{R=1}^{R_{max}}\epsilon_{i,R}, u_i \in \mathcal{U}\}$, $\Delta' = \{ \delta_i | \delta_i = \sum_{R=1}^{R_{max}} \pi_{i,R} \delta, u_i \in \mathcal{U} \}$. Note that any $\epsilon_{i,R} \in \Phi_R$ is less than or equal to the privacy budget of $u_i$ remaining at the beginning of the round $R$. Because of this property, for all $\epsilon_i \in \Phi$, $\epsilon_i - \sum_{R=1}^{R_{max}}\epsilon_{i,R} \geq 0$. Also,  $R_{max} \delta \geq \sum_{R=1}^{R_{max}}\pi_{i,R} \delta$.

Therefore, {\em PDP-SGD} satisfies $(\Phi,\Delta)$-PDP, where $\Phi=\{\epsilon_1,...,\epsilon_n\}$ and $\Delta=\{\delta_i | \delta_i = R_{max}\delta, i\in \{1,...,n\}\}$.
\end{proof}
We note that each $\delta_i$ increases as there are more rounds, but it is bounded by $R_{max} \delta$ (actually, $\sum_{R=1}^{R_{max}}\pi_{i,R} \delta$). In our experiments, we set $R_{max}$ to be at most 10, so each $\delta_i$ is not large either.

\subsection{RDP Accountant in PDP-SGD}
\label{sec:rdptopdp}

\paragraph{RDP Accountant}

Most privacy accounting methods rely on Renyi Differential Privacy (RDP)~\citep{mironov2017renyi}, which directly bounds the divergence between two distributions.
\begin{definition}
\label{def:rdp}
($(\alpha,\rho)$-Renyi Differential Privacy) A randomized mechanism $\mathcal{M}:\mathcal{D}\rightarrow\mathcal{R}$ is said to have $(\alpha,\rho)$-RDP if for any adjacent $D,D'\in\mathcal{D}$ it holds that
\begin{equation}
    D_\alpha (\mathcal{M}(D)||\mathcal{M}(D')) =\frac{1}{\alpha-1}log E_{o\sim \mathcal{M}(D)}\left[\left(\frac{\Pr[\mathcal{M}(D)=o]}{\Pr[\mathcal{M}(D')=o]}\right)^{\alpha-1}\right] \leq \rho
\end{equation}
\end{definition}

\paragraph{From Individual RDP to PDP} We also cover how to convert individual privacy loss to $(\Phi,\Delta)$-PDP. We introduce the concept of individual RDP and propose the conversion of it to $(\Phi, \Delta)$-PDP. There are several recent works on individual privacy accounting~\citep{feldman2021individual,da2022individual}, and the proposed conversion can be used in these methods.

\citet{feldman2021individual} suggest individual RDP (iRDP) to calculate individual privacy loss. They measure the maximum possible effect of an individual data point on a dataset using this concept. 

\begin{definition}
\label{def:irdp}
(Individual RDP). Fix $m\in\mathbb{N}$ and a data point $X$. A randomized mechanism $\mathcal{M}$ satisfies $(\alpha,\Gamma)$-individual RDP for $X$ if for all datasets $D=(X_1,...,X_n)$ such that $n\leq m$ and $X_i=X$ for some $i$, it holds that
\begin{equation*}
    D_\alpha (\mathcal{M}(D)||\mathcal{M}(D^{-i})) \leq \rho_i,\;\mathrm{where}\;\Gamma=\{\rho_1,...,\rho_n\}.
\end{equation*}
\end{definition}

We show that this iRDP can be converted to the proposed $(\Phi,\Delta)$-PDP in the following theorem:

\begin{restatable}[From Individual RDP to $(\Phi,\Delta)$-PDP]{theorem}{rdpthm}
\label{thm:rdp}
    If $\mathcal{M}$ is an $(\alpha,\Gamma)$-iRDP mechanism, it also satisfies $(\Phi,\Delta)$-PDP where $\Phi=\{\epsilon_i|\epsilon_i=\rho_i+\frac{log1/\delta_i}{\alpha-1}, \,\rho_i\in\Gamma, \,\delta_i\in\Delta, \,i\in [1...n]\}$, $\Delta=\{\delta_1,...,\delta_n\}$, and $\Gamma=\{\rho_1,...,\rho_n\}$.
\end{restatable}

The proof can be found in Section~\ref{appendix:individualrdptopdp}.





\section{Experiments}
\label{sec:experiments}
In this section, we compare the proposed method with a conventional DP-SGD and simple combinations of DP-SGD with existing PDP mechanisms and answer the following questions: (1) How does our algorithm compare in utility with various baselines? (2) How does the algorithm perform against different hyperparameters, $\epsilon$ distributions, and intervals? (3) How efficient are the algorithms?

\subsection{Experimental Setup}
\label{sec:experimentalsetup}

\paragraph{Datasets}

We use the following real benchmark datasets: (1) MNIST~\citep{deng2012mnist} contains images of handwritten digits, (2) SVHN~\citep{37648} contains images of street view house numbers, and (3) Fashion-MNIST~\citep{xiao2017/online} contains images of different types of clothes.

\paragraph{Baselines}

We compare {\em PDP-SGD} with conventional DP-SGD, which guarantees the same $\epsilon$ for the entire data. We assume that users have individual privacy requirements $(\epsilon_i\in\Phi $ and $ \delta_i\in\Delta)$ and we want to satisfy all user's privacy constraints. We also compare naive combinations with the existing PDP mechanisms ({\em Sampling}~\citep{jorgensen2015conservative} and {{\em AdaPDP}~\citep{niu2021adapdp}). These also satisfy $(\Phi,\Delta)$-PDP. The detailed algorithms and explanations are in Sections~\ref{appendix:baseline1} and~\ref{appendix:baseline2}. 

\begin{figure}[t]
\centering
    \includegraphics[scale=0.43]{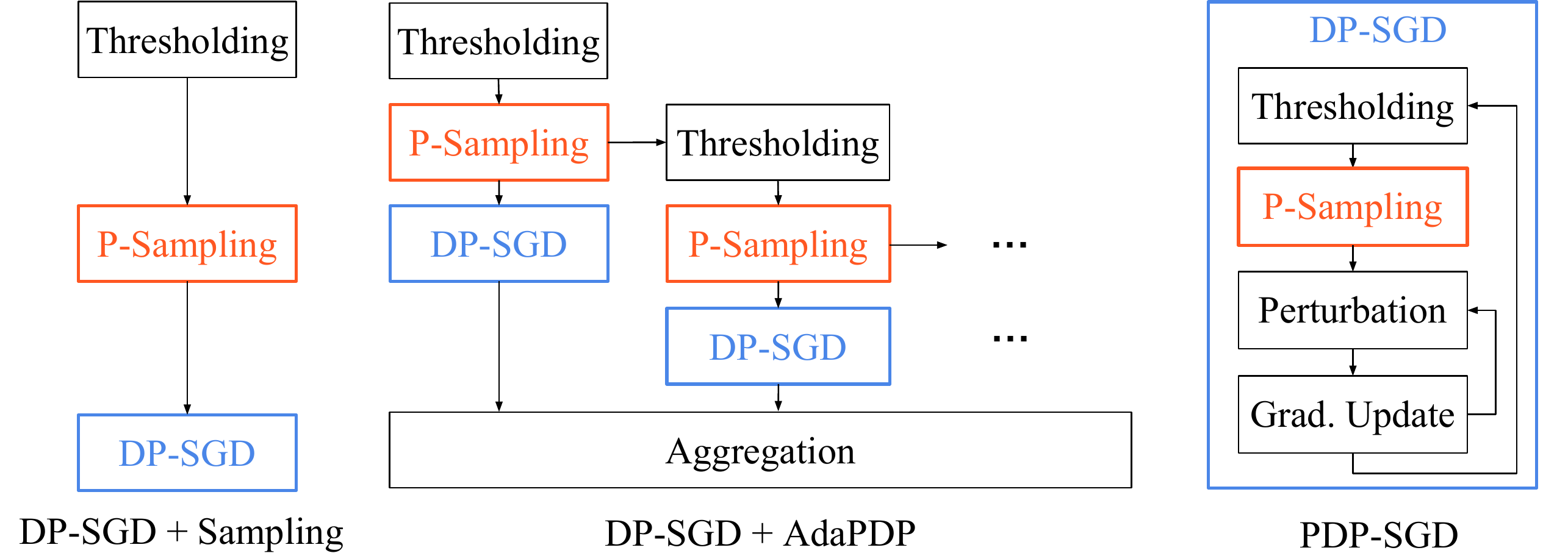}
    \caption{Comparison of the two baselines and {\em PDP-SGD} algorithms satisfying ($\Phi$, $\Delta$)-PDP.}
    \vspace{-0.4cm}
\label{fig:overviews}
\end{figure}

For all models, we use simple convolutional neural networks (CNN). The details on the model architectures can be found in Section~\ref{appendix:modelarchitecture}. We use PyTorch~\citep{pytorch} with Opacus~\citep{opacus}. All experiments are performed with NVIDIA Titan RTX GPUs, and all models are evaluated on separate test sets. All experiments are repeated 3-5 times with different random seeds.

\subsection{Accuracy Results} 
\label{sec:acc}

\begin{table}[t]
\setlength{\tabcolsep}{3pt}
\caption{Test accuracies of the models trained on conventional DP-SGD, PDP baselines, and {\em PDP-SGD} on the three datasets. We evaluate using three different $\epsilon$ distributions in Figure~\ref{fig:skew}.}
\begin{center}
\scalebox{0.9}{
\begin{tabular}{c?c?c|c|c|cc}
\toprule
\multirow{2}{*}{Dataset} & \multirow{2}{*}{Skew} & \multirow{2}{*}{DP-SGD} & DP-SGD & DP-SGD & \multicolumn{2}{c}{{\em PDP-SGD}} \\
 & & & + Sampling & + AdaPDP & Fixed ($loss_f$) & Adaptive ($loss_a$) \\
\midrule
\multirow{3}{*}{\begin{tabular}{c} SVHN \end{tabular}} & $k=-0.2$ & 63.95 $\pm$ 3.78 & 69.54 $\pm$ 1.00 & 69.55 $\pm$ 1.00 & { 73.87 $\pm$ 0.89} & {\bf75.66 $\pm$ 1.19} \\
& $k=0$ & 63.95 $\pm$ 3.78 & 72.64 $\pm$ 0.48 & 72.84 $\pm$ 0.60 & { 76.79 $\pm$ 1.14} & {\bf 77.50 $\pm$ 0.64}\\
& $k=0.2$ & 63.95 $\pm$ 3.78 & 73.19 $\pm$ 0.78 & { 73.24 $\pm$ 0.73} & 71.97 $\pm$ 2.55 & {\bf 77.73 $\pm$ 0.34} \\
\midrule
\multirow{3}{*}{\begin{tabular}{c} Fashion- \\ MNIST \end{tabular}} & $k=-0.2$ & 75.20 $\pm$ 0.73 & 76.44 $\pm$ 0.39 & 76.43 $\pm$ 0.40 & { 78.36 $\pm$ 0.86} & {\bf 78.82 $\pm$ 0.47}\\
& $k=0$ & 75.20 $\pm$ 0.73 & 77.42 $\pm$ 0.23 & 77.58 $\pm$ 0.19 & { 79.46 $\pm$ 0.49} & {\bf 79.81 $\pm$ 1.01 } \\
& $k=0.2$ & 75.20 $\pm$ 0.73 & 78.23 $\pm$ 0.65 & { 78.30 $\pm$ 0.59} & 78.08 $\pm$ 1.34 & {\bf 80.29 $\pm$ 0.31 } \\
\midrule
\multirow{3}{*}{\begin{tabular}{c} MNIST \end{tabular}} & $k=-0.2$ & 90.42 $\pm$ 2.78 & 92.57 $\pm$ 0.13 & 92.57 $\pm$ 0.16 & {\bf 94.65 $\pm$ 0.35} & 94.54 $\pm$ 0.57 \\
& $k=0$ & 90.42 $\pm$ 2.78 & 92.50 $\pm$ 0.55 & 92.78 $\pm$ 0.53 & {\bf 95.19 $\pm$ 0.27} & 95.05 $\pm$ 0.44 \\
& $k=0.2$ & 90.42 $\pm$ 2.78 & 92.48 $\pm$ 0.39 & 92.52 $\pm$ 0.42 & {\bf 95.20 $\pm$ 0.15} & 94.86 $\pm$ 0.24\\

\bottomrule
\end{tabular}
}
\end{center}
\vskip -0.2in
\label{tbl:skew}
\end{table}


We compare the model accuracy results of our algorithm against baselines. For all the experiments, we use $\epsilon$ values within the range ($0.5 \leq \epsilon \leq 1.0$). We fix the minimum $\epsilon$ value to be 0.5 so that DP-SGD will train using the strictest privacy $\epsilon = 0.5$ in order to satisfy all the user privacy constraints.

We demonstrate the key advantage of PDP algorithms, which is being able to support individual $\epsilon$ values. We generate $\epsilon$ values using an exponential distribution and vary the exponent parameter ($k$ value) to adjust skewness. The level of granularity in these distributions may not be realistic, but it allows us to clearly observe how the algorithm's results change in extreme cases. As shown in Figure~\ref{fig:skew}, we choose three $k$ values -0.2, 0.0, and 0.2 in order to evaluate our algorithms on uniform and skewed distributions. We do not use $k$ values larger than 0.2 as the distribution is skewed to the extent that no user has $\epsilon = 0.5$, which is not consistent with our experimental setup regarding the $\epsilon$ range. Likewise, we do not use $k$ values smaller than -0.2 for similar reasons.

\begin{figure}[t]
\centering
    \includegraphics[scale=0.27]{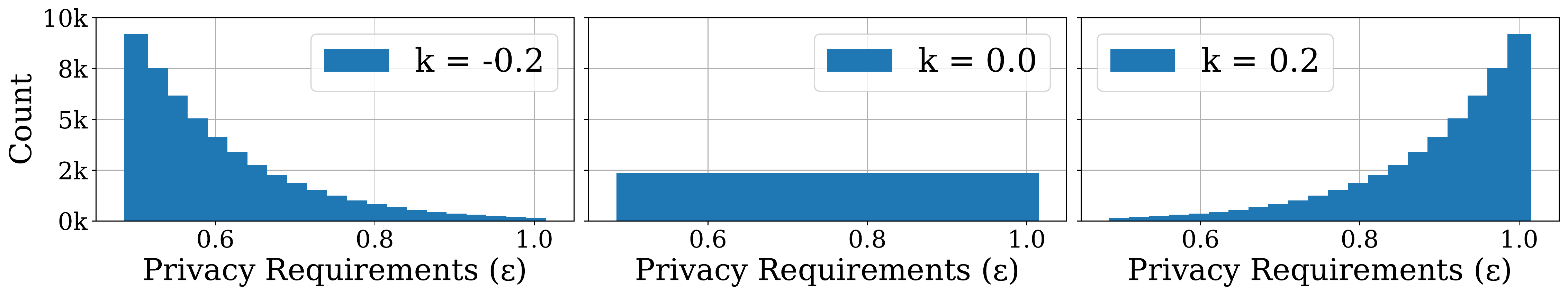}
   \vspace{-0.2cm}
    \caption{We consider three scenarios for user privacy requirements by generating $\epsilon$ values following an exponential distribution (Count($\epsilon$) = $n_{train}(c_1 e^{k\epsilon}+c_2)$, constant values are in Section~\ref{appendix:epsdist})}
    
\label{fig:skew}
\vspace{-0.6cm}
\end{figure}

\paragraph{Varying Weights $w_1$, $w_2$ of $loss_f$}

We evaluate the accuracies of {\em PDP-SGD} using different weight values in Figure~\ref{fig:weight}. We change $w_1$ from 0.3 to 0.9, and $w_2$ is determined by $w_1$. For all $k$ values, the performances increase as $w_1$ increases. However, the result shows similar average accuracies and larger variances when $w_1 \geq 0.7$ since the resulting threshold $\tau$ values do not change much.

\begin{figure}[t]
\centering
    \includegraphics[scale=0.27]{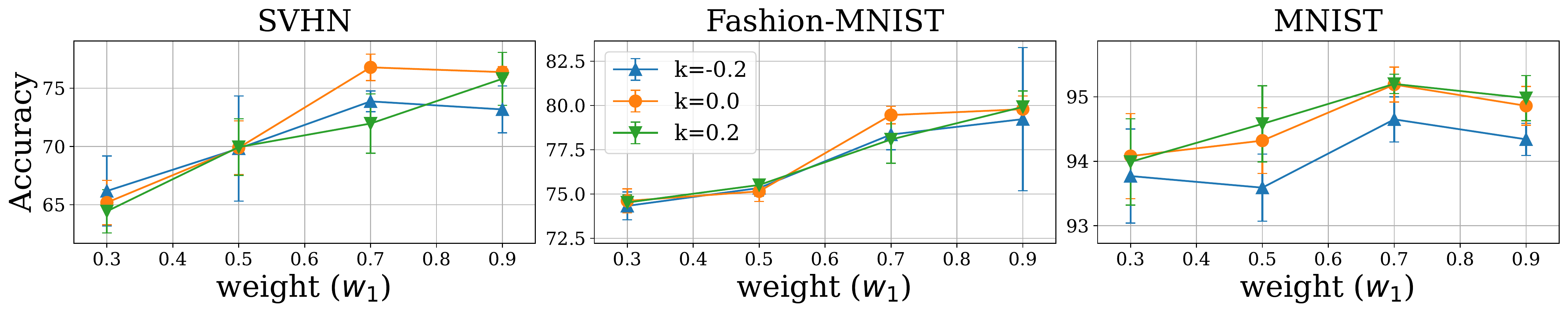}
    \vspace{-0.2cm}
    \caption{Test accuracies of the models trained {\em PDP-SGD} on the three datasets for different $w_1$.}
    \vspace{-0.2cm}
\label{fig:weight}
\end{figure}


\paragraph{Varying Privacy Requirements}
We evaluate the accuracies of {\em PDP-SGD} using different privacy requirements as shown in Figure~\ref{fig:interval}. We use varying $\epsilon$ ranges [0.5, 0.5 + $\alpha$] where the interval $\alpha$ ranges from 0.1 to 0.9 and set $k = 0$ for the exponential distribution of $\epsilon$ values. As the intervals increase, the $\epsilon$ values increase as well, which means that the accuracy has more room to improve as the privacy constraints are more relaxed. The three plots thus show overall increasing trends of test accuracy.
\begin{figure}[t]
\centering
    \includegraphics[scale=0.27]{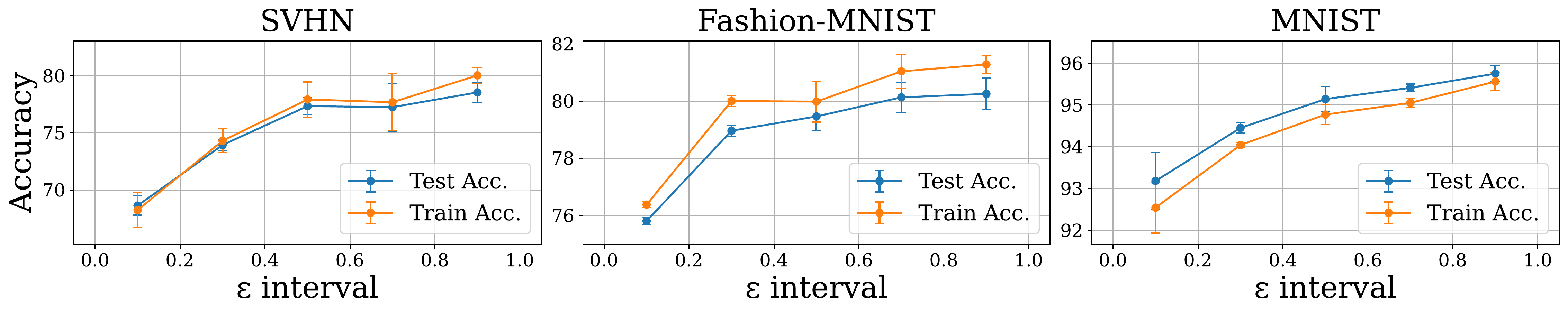}
    \vspace{-0.2cm}
    \caption{Test and train accuracies of the models trained {\em PDP-SGD} on the three datasets for different $\epsilon$ intervals. A larger interval means the $\epsilon$ values can be large, resulting in higher test accuracy.}
\label{fig:interval}
\end{figure}

\paragraph{Comparison with baselines}

Table~\ref{tbl:skew} shows that all PDP algorithms outperform DP-SGD because they are able to effectively use their privacy budgets without degrading accuracy unnecessarily. Among the algorithms, {\em PDP-SGD} tends to have the best accuracies, but also the largest average standard deviations due to more frequent sampling processes than {\em Sampling} and {\em AdaPDP}. We also compare using $loss_f$ (Eqn.~\ref{eqn:floss}) and $loss_a$ (Eqn.~\ref{eqn:aloss}). We observe that using $loss_a$ results in stable performance across different $k$ values. Using $loss_a$ automatically adjusts the weights $w_i$, but the resulting model performance is not always the best depending on the dataset.





\subsection{Efficiency}
\begin{wraptable}{r}{7.cm}
\vspace{-0.4cm}
\caption{Averaged runtimes and the number of iterations and model parameters of the algorithms.}
\vspace{-0.2cm}
\begin{center}
\scalebox{0.9}{
\begin{tabular}{l|ccc}
\toprule
  & Sampling & AdaPDP & {\em PDP-SGD}\\
\midrule
     Runtime (s) & 1086.84  & 1192.90 & 772.20 \\
     \# Iteration & 37,695 & 38,961 & 23,746 \\
     \# Parameter & 38,009 & 72,908 & 38,009 \\
\bottomrule
\end{tabular}
}
\end{center}
\label{tbl:runtime}
\vspace{-0.5cm}
\end{wraptable}
We compare the average runtimes, the average number of total iterations and model parameters of the three algorithms on three datasets in Table~\ref{tbl:runtime}, as it is important to guarantee privacy using less computation. We use the default $\epsilon$ range [0.5, 1] and set $k = 0$ for the distribution of values. As a result, compared to {\em Sampling} and {\em PDP-SGD}, using multiple rounds of iterations does add some overhead to the overall runtime in {\em AdaPDP}, which also requires multiple seperated models. {\em PDP-SGD} shows shorter runtimes than the other algorithms because it uses approximate Poisson sampling with weighted random sampling, which can reduce computation.

\subsection{Comparison with Individualized PATE}
\begin{wraptable}{r}{7.4cm}
\vspace{-0.4cm}
\caption{Test accuracies of the models trained with {\em PDP-SGD} and Individualized PATE (I-PATE).}
\vspace{-0.2cm}
\scalebox{0.9}{
\begin{tabular}{c|c|cc|c}
\toprule
  \multirow{2}{*}{Dataset} & \multirow{2}{*}{Setup} & \multicolumn{2}{c|}{I-PATE}  & \multirow{2}{*}{{\em PDP-SGD}} \\
  & & Upsample & Weight & \\  
\midrule
     \multirow{2}{*}{MNIST} & Case1 & 97.20 & 97.24 & 96.34 \\
     &  Case2 & 97.18 & 97.33 & 96.05 \\
\midrule
     \multirow{2}{*}{SVHN} & Case1 & 64.33 & 66.40 & 74.25 \\
     &  Case2 & 62.85 & 65.65 & 75.45 \\
\bottomrule
\end{tabular}
}
\label{tbl:ipate}
\vspace{-0.3cm}
\end{wraptable}

We compare our algorithm with Individualized PATE~\citep{muhl2022personalized}. We note that comparing DP-SGD and PATE methods is challenging because they require different settings (e.g., PATE uses a public dataset, whereas DP-SGD does not). There is an existing work~\citep{uniyal2021dp} that makes a comparison, but only in terms of fairness. In order to make a fair comparison, we adapt the setting of \citet{liu2020intrinsic} and use PATE's public dataset to pre-train our model with non-private SGD before running our {\em PDP-SGD} algorithm. We also use the dataset and student model architecture in \citet{muhl2022personalized}. We use three privacy groups with different $\epsilon$ values ($1.0,\;2.0,\;3.0$), and we consider two cases of the percentages of the groups (Case1: [34\%, 43\%, 23\%], Case2: [54\%, 37\%, 9\%])). Table~\ref{tbl:ipate} shows that \ours{} outperform I-PATE on SVHN and has comparable performance on MNIST. The reason I-PATE underperforms on SVHN is that the SVHN task is more difficult than the MNIST task, so training separate teacher models on separate partitions is less likely to lead to accurate results. In comparison, \ours{} performs better as it trains on the entire dataset.





\section{Conclusion}

We introduced how to extend $(\epsilon,\delta)$-DP mechanisms to the $(\Phi,\Delta)$-PDP using the personalized sampling mechanism. We proposed an ML training algorithm called \ours{} based on DP-SGD embedding the personalized sampling within the internal process. In addition, we showed that \ours{} theoretically satisfies $(\Phi,\Delta)$-PDP and provided a conversion from individual RDP to $(\Phi,\Delta)$-PDP for individual privacy accounting. The experimental results showed that our algorithms significantly outperform DP-SGD and simple combinations with the existing PDP mechanisms on real data benchmarks.  

\paragraph{Societal Impact \& Limitation} 
Our sampling-based mechanism is effective for personalized DP, but can also be viewed as adding bias to the training data. As a result, the utility results may vary by person, which may be perceived as being discriminating. Further adjusting our sampling to also mitigate bias is an interesting future work.

\bibliography{main}
\bibliographystyle{abbrvnat}

\clearpage
\appendix

\section{Variants of Differential Privacy}
While DP is used in various applications, it also has shortcomings, so several variants of DP have been proposed. 

Local Differential Privacy (LDP)~\citep{duchi2013local} is the most studied concept other than DP. In the LDP literature, traditional DP is called centralized DP, and client-side and server-side protections are considered separately. The main premise of centralized DP is that the data server can be trusted, but in real life, there are many cases where privacy must be protected at the client level. The concept that came out from this setup is LDP, which is stricter than conventional DP. In the ML field, LDP is most widely used in Federated Learning~\citep{truex2020ldp,zhao2020local,sun2020ldp,kim2021federated}.

Individual Differential Privacy (IDP)~\citep{soria2017individual} is also a relaxation of DP in order to improve utility. The key approach is to assume a given dataset $D$ and guarantee DP only on that dataset, which is referred to as ($\epsilon, D$)-IDP. Since the focus is on a single dataset, less noise is generated compared to DP, and the definition is more intuitive. Since IDP relaxes DP, it still uses a single $\epsilon$ value and is thus fundamentally different than PDP, which is the focus of this work.

Label Differential Privacy (LabelDP)~\citep{ghazi2021deep} is also a relaxed definition for improving utility, and assumes a situation where only labels must be protected in the process of training an ML model. The Randomized Response (RR) technique that satisfies this LabelDP shows better performance than the conventional RR~\citep{warner1965randomized}.

\section{Baselines}

Continuing from Section~\ref{sec:experimentalsetup}, we conduct experiments with 2 baselines, {\em Sampling} and {\em AdaPDP}, that are naively combined with the existing PDP mechanisms. The detailed mechanisms are explained in the following sections. 

\subsection{DP-SGD with Sampling Mechanism}
\label{appendix:baseline1}

We na\"ively apply the sampling mechanism to DP-SGD as a baseline. The existing sampling mechanism guarantees $\Phi$-PDP extending $\epsilon$-DP mechanism. In this work, we showed that the sampling mechanism can extend $(\epsilon,\delta)$-DP mechanisms in Section~\ref{sec:alg}. We thus can use DP-SGD with Sampling mechanism as a baseline that satisfies $(\Phi,\Delta)$-PDP. The detailed algorithm is defined as follows:

\begin{definition}
(DP-SGD with Sampling) Let $A_{DP}^{\tau,\delta}$ denote a model training algorithm using DP-SGD and satisfying ($\tau, \delta$)-DP. When the model is trained on a dataset $D\subset \mathcal{D}$, let $f_{s}(D,\Phi,\tau)$ denote the preprocessing step that independently samples each data point $x_i \in D$ with probability
\begin{equation}
\label{eq:sam}
    \pi_i=\left\{
    \begin{matrix}
\frac{e^{\epsilon_i}-1}{e^{\tau}-1} \;\;\;\; \text{if}\;\; \epsilon_i<\tau\\ 
\;\;\;1\;\;\;\;\;\text{otherwise}
\end{matrix}\right.
\end{equation}
where $\min \Phi \leq \tau \leq \max \Phi $ is a configurable threshold with privacy budgets $\Phi=\{\epsilon_1,...,\epsilon_n\}$. Finally, the mechanism $\mathcal{M}_s$ is defined as
\begin{equation}
    \mathcal{M}_{s}(D,\Phi,\tau) = A_{DP}^{\tau,\delta}(f_{s}(D,\Phi,\tau)).
\end{equation}
\end{definition}

\subsection{DP-SGD with AdaPDP}
\label{appendix:baseline2}

We naively apply the multi-round sampling mechanism proposed by \citet{niu2021adapdp} to the DP-SGD algorithm as a baseline. Algorithm~\ref{alg:ada} shows the entire procedure. 

Before running DP-SGD, there are two tasks: computing the threshold and sampling. The threshold is set in the range of $\min \Phi \leq th \leq \max \Phi$, and among these values, the value with the lowest estimated utility loss is selected. We estimate the loss using the following equation, which is a weighted sum of two losses:
\begin{equation*}
\vspace{-0.2cm}
    \begin{aligned}
        loss(\Phi,\,\tau)=w_1\cdot\sum_{i:\epsilon_i<t,\epsilon_i \in \Phi} \epsilon_i(1-\frac{e^{\epsilon_i}-1}{e^{\tau}-1})+w_2\cdot\sum_{i:\epsilon_i>t,\epsilon_i \in \Phi}(\epsilon_i-\tau)
    \end{aligned}
\end{equation*}
where the weights $w_i$ are parameters for controlling the importance of each loss. 

The first two losses penalize the budget wastes of the two cases mentioned in the previous paragraph. The first one is when the data point is not sampled, and in this case, it is set to be proportional to $1-\pi_i$.  The second loss is the waste caused by the threshold after sampling and is computed as the sum of $\epsilon_i-\tau$. 

After deriving the minimum threshold $\tau$, we use it to run the sampling mechanism.

Next, we run DP-SGD to satisfy $\tau_T$ for the sampled dataset $D_T$, the remaining budget is calculated by subtracting the actual $\epsilon$ value used from the total budget $\Phi$. Then, the set of models trained in all rounds are returned and a weighted sum is performed on the models' outputs using the threshold and the size of $D_T$ as follows: 
$\theta_{ada}(x) = \sum_T \tau_T |D_T|\theta_T(x) / \sum_T \tau_T |D_T|$. The weighted sum effectively weights the parameters of each model in proportion to the dataset size used for its training. The PrivacyAccounting function can be any accounting function, and we use the RDP accountant.

Figure~\ref{fig:ada} shows how the privacy requirements and sampling probabilities change at each round.
\begin{figure}[t]
\centering
    \includegraphics[scale=0.33]{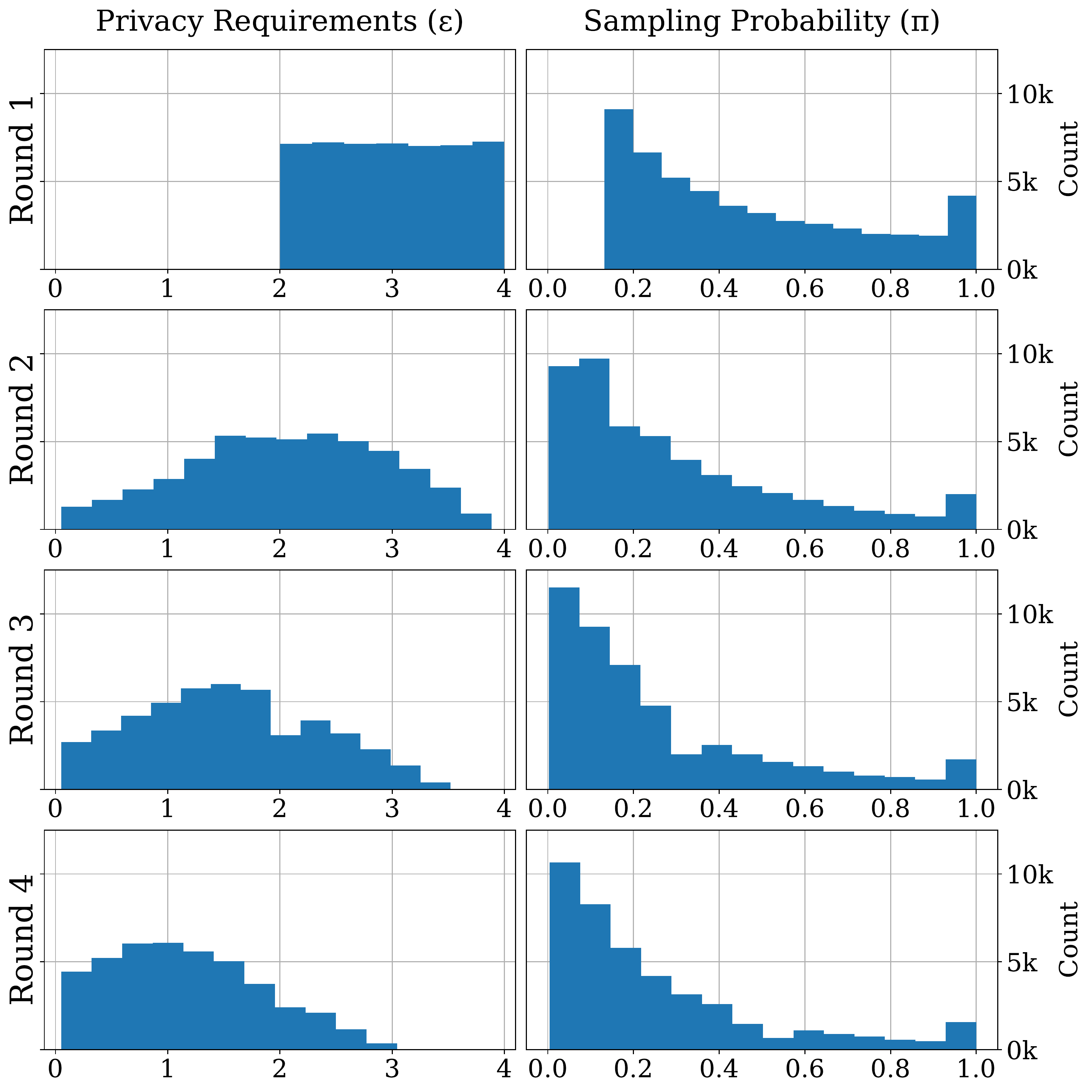}
    \caption{Privacy requirements ($\epsilon$) and sampling probability ($\pi$) when running the {\em AdaPDP} algorithm. For each round, the left figure shows user counts for different $\epsilon$ values, and the right figure user counts for different $\pi$ values.}
    \vspace{-0.2cm}
\label{fig:ada}
\end{figure}

\begin{algorithm}[t]
\caption{DP-SGD with AdaPDP}
\label{alg:ada}
\begin{algorithmic}
\Require Examples $D=\{x_1,...,x_n\}$, loss function $l(\theta)=\frac{1}{n}\Sigma_i l(\theta, x_i)$. Parameters: learning rate $\mu_t$, gradient norm bound $C$, tolerance $\beta$, privacy budgets $\Phi=\{\epsilon_1,...,\epsilon_n\}$.

\For{$T=1,2,...,k$}

\State $\tau_T \gets \min(loss(\Phi,\,\tau)$,  $\min \Phi \leq \tau \leq \max \Phi$) 
\State $D_T \gets$ Sampling($D$, $\Phi$, $\tau_T$) 
\State $t \gets 0$
\While{$\left \| \epsilon-\tau_T \right \| \leq \beta$}
\State $g_t(x_i) \gets \nabla_\theta l(\theta_t ; x_i), \forall x_i \in D_T$ 
\State $\bar{g_t}(x_i) \gets g_t(x_i)\cdot \min (1,\frac{C}{\left\|g_t(x_i)\right\|_2}), \forall x_i \in D_T$ 
\State $\tilde{g_t} \gets \frac{1}{n} \sum_{i=1}^{n}(\bar{g_t}(x_i)+N(0,\sigma^2C^2\mathbb{I}))$ 
\State $\theta_{t+1} \gets \theta_t-\mu_t\tilde{g_t}$ 
\State $\epsilon \gets \epsilon+$ PrivacyAccounting$(\sigma,\,k,\,t)$
\State $t \gets t+1$
\EndWhile
\State $\epsilon_i \gets \max(\epsilon_i-\epsilon, \,0)$, $\forall i \in \{i|x_i \in D_T\}$ 
\State $\theta_T \gets \theta_t$
\EndFor

\Ensure $\{\theta_1,...,\theta_k\}$

\end{algorithmic}
\end{algorithm}

\section{Implementation Details}
\label{appendix:implementation}

\subsection{$\epsilon$ distribution}
\label{appendix:epsdist}

In Section~\ref{sec:acc}, we generate $\epsilon$ values using an exponential distribution and vary the exponent parameter ($k$ value) to adjust skewness. The level of granularity in these distributions may not be realistic, but it allows us to clearly observe how the algorithm's results change in extreme cases. We set the number of privacy groups as 20 and assign an $\epsilon$ value for each group. The number of individual data points in the group with the corresponding $\epsilon$ value follow $\mathrm{Count}(\epsilon)$ = $(c_1 e^{k\epsilon}+c_2) \cdot n_{train}$ where $n_{train}$ is the number of training data points. The constant values are $c_1=2.098,\;c_2=-1.715$ when $k=-0.2$ and $c_1=1.554,\;c_2=-1.715$ when $k = 0.2$. Count($\epsilon$) becomes a constant function when $k=0.0$. In other words, the $\mathrm{Count}(\epsilon)$ follows the distribution applying scaling and parallel shift to $e^{k\epsilon}$ to move privacy groups into the range $[0.5,1]$ and make the total number of data points in the groups same as $n_{train}$ maintaining the exponential distribution.

\subsection{Model architecture and Parameters}
\label{appendix:modelarchitecture}

Continuing from Section~\ref{sec:experimentalsetup}, we explain the two types of CNN architectures used for the three datasets. Table~\ref{tbl:architecture1} shows the architecture for the SVHN dataset, and Table~\ref{tbl:architecture2} shows that for the MNIST and Fashion-MNIST dataset. $K$, $S$, and $P$ represent the size of the kernel, stride, and padding, respectively. We set the learning rate to 0.05 and the batch size to 64. For the DP-SGD parameters, we set the initial noise multiplier to 1.3 and the maximum gradient norm to 1.0.

\begin{table}[h]
\caption{CNN architecture for the SVHN dataset.}
\begin{center}
\begin{small}
\begin{tabular}{c|c|c|c}
\toprule
    Layer & Type of Layer & Parameters & Activation\\
\midrule
    1 & Convolutional & 6, k:5, s:1, p:0 & ReLU \\
    2 & Max Pooling & size (2,2) & - \\
    3 & Convolutional & 16, k:5, s:1, p:0  & ReLU \\
    4 & Max Pooling & size (2,2) & - \\
    5 & Flatten & - & - \\
    6 & Fully Connected & 120 nodes & ReLU \\
    7 & Fully Connected & 84 nodes & ReLU \\
    8 & Fully Connected & 10 nodes & softmax \\
\bottomrule
\end{tabular}
\end{small}
\end{center}
\label{tbl:architecture1}
\end{table}

\begin{table}[h]
\caption{CNN architecture for the MNIST and Fashion-MNIST datasets.}
\begin{center}
\begin{small}
\begin{tabular}{c|c|c|c}
\toprule
    Layer & Type of Layer & Parameters & Activation\\
\midrule
    1 & Convolutional & 16, k:8, s:2, p:3 & ReLU \\
    2 & Max Pooling & size (2,2) & - \\
    3 & Convolutional & 32, k:4, s:2, p:0  & ReLU \\
    4 & Max Pooling & size (2,2) & - \\
    5 & Flatten & - & - \\
    6 & Fully Connected & 32 nodes & ReLU \\
    7 & Fully Connected & 10 nodes & softmax \\
\bottomrule
\end{tabular}
\end{small}
\end{center}
\label{tbl:architecture2}
\end{table}

When computing thresholds in {\em Sampling} and {\em AdaPDP}, we used different weight values.
We used $w_1=0.2,\,w_2=0.8$ for {\em AdaPDP}, and $w_1=0.7,\,w_2=0.3$ for {\em PDP-SGD} with $loss_f$. In general, these weights can be set using cross validation. With $loss_a$, since the weight values are automatically adjusted, we do not have to tune the parameters.

\subsection{PDP Sampler}
\label{appendix:pdpsampler}

PDP Sampler is an extended version of uniform sampler of Opacus~\cite{opacus}. This sampler gets the list of each data's sampling probability and takes Poisson sampling with respect to the sampling probabilities. The detailed algorithm is in Algorithm~\ref{alg:pdpsampler}. The $random()$ function indicates the random number generator in the range $[0,1]$ in Algorithm~\ref{alg:pdpsampler}.

\begin{algorithm}[t]
\caption{PDP Sampler}
\label{alg:pdpsampler}
\begin{algorithmic}
\Require privacy requirements $\Phi$, number of samples $n_{sample}$, number of batches $n_{batch}$.

\While{$n_{batch} > 0$}
\State $y \gets [y_i|y_i=random(0,1),\;i=[1,...,n_{sample}]]$
\State $M \gets [m_i|m_i=\mathbbm{1}[x_i<y_i],\;x_i \in \Phi,\;y_i\in y,\;i=[1,...,n_{sample}]]$
\State $idx \gets [i|m_i \neq 0,\;m_i\in M,\;i=[1,...,n_{sample}]]$
\State yield $idx$
\State $n_{batch} \gets n_{batch}-1$
    
\EndWhile

\end{algorithmic}
\end{algorithm}

\subsection{getNoiseMultiplier function}
\label{appendix:getnoisemltplr}

In Opacus~\cite{opacus}, there is getNoiseMultiplier function using privacy accountant and binary search. We extended the function by adding the history update, so the function can contain the information about previous round's accountant. The details are in Algorithm~\ref{alg:getnoisemltplr}. The history $h$ in the input of Algorithm~\ref{alg:getnoisemltplr} means the history list for privacy accounting, which records noise multiplier, sample rate, and steps. $\mathrm{PrivacyAccounting}$ can be any privacy accountant including RDP~\citep{mironov2017renyi}, GDP~\citep{bu2020deep}, or PRV~\citep{gopi2021numerical}, to produce $\epsilon'$ (Step 10). For our implementation, we use an RDP accountant. 

\begin{algorithm}[t]
\caption{Extended getNoiseMultiplier function}
\label{alg:getnoisemltplr}
\begin{algorithmic}
\Require target privacy budget $\epsilon_t$, $\delta_t$, tolerance $\beta$, sample rate $q$, the number of iterations $k$, history $h$

\State $\epsilon_{high} \gets \infty$
\State $\sigma_{low},\;\sigma_{high}\gets 0,\;10$

\While{$\epsilon_{high}>\epsilon_t$}
\State $\sigma_{high}\gets2\sigma_{high}$
\State $h \gets h + [(\sigma_{high}, q, k)]$
\State $\epsilon_{high}\gets\mathrm{PrivacyAccounting}(h,\;\delta_t)$
\EndWhile

\While{$\epsilon_t-\epsilon_{high}>\beta$}
\State $\sigma\gets(\sigma_{low}+\sigma_{high})/2$
\State $h \gets h + [(\sigma, q, k)]$
\State $\epsilon_{high}\gets\mathrm{PrivacyAccounting}(h,\;\delta_t)$
\EndWhile

\If{$\epsilon < \epsilon_t$}
\State $\sigma_{high},\;\epsilon_{high} \gets \sigma,\;\epsilon$
\Else
\State $\sigma_{low} \gets \sigma$
\EndIf

\Ensure $\sigma_{high}$

\end{algorithmic}
\end{algorithm}

\section{Proofs}

\subsection{Proof of Theorem \ref{thm:sam}}
\label{appendix:samtheoremproof}

Continuing from Section~\ref{sec:prove}, we prove the following theorem.

\samthm*

\begin{proof}
Let $x$ be a data record related to an arbitrary user $u_i$ and $u_i$ has a privacy preference of $\epsilon_i$.
For convenience, we simplify several notions as follows: $f_s(D)=f_s(D,\Phi,\tau)$, $\mathcal{M}_s(D)=\mathcal{M}_s(D,\Phi,\tau)$, and $A(Z) = A^{\tau,\delta}_{DP}(Z)$.

To prove that $\mathcal{M}_s$ satisfies Definition 2.3, we need to prove two cases of inequalities in Definition 2.3: i) $x \in D$ and $x \not\in D'$ ($D'=D_{-x}$) , ii) $x \in D'$ and $x \not\in D$ ($D=D_{-x}$ and $D'=D$).

i) $Pr[{\mathcal{M}_{s}}(D)\in O] \leq e^{\epsilon_i} Pr[\mathcal{M}_{s} (D_{-x}) \in O] + \pi_i \delta$
\begin{flalign*}
    Pr[\mathcal{M}_{s}(D)\in O] &= \sum_{Z\subseteq D_{-x}}\pi_i Pr[f_{s}(D_{-x})=Z]Pr[A(Z_{+x})\in O]&&\\
    &\;\;\;\;+\sum_{Z\subseteq D_{-x}}(1-\pi_i) Pr[f_{s}(D_{-x})=Z]Pr[A(Z)\in O]&&\\
    &\leq \sum_{Z\subseteq D_{-x}}\pi_i Pr[f_{s}(D_{-x})=Z](e^{\tau}Pr[A(Z)\in O]+\delta)&&\\
    & \;\;\;\;+ (1-\pi_i) Pr[\mathcal{M}_{s}(D_{-x})\in O]&&\\
    &= (1-\pi_i+\pi_i e^{\tau}) Pr[\mathcal{M}_{s}(D_{-x})\in O]+\delta\sum_{Z\subseteq D_{-x}}\pi_i Pr[f_{s}(D_{-x})=Z]&&\\
    &\leq e^{\epsilon_i} Pr[\mathcal{M}_{s} (D_{-x}) \in O] + \pi_i \delta&&
\end{flalign*}

Here, the first equality divides $Pr[\mathcal{M}_s(D)\in O]$ into two cases: whether the data $x$ is sampled or not. The first inequality is due to the fact that DP-SGD $A^{\tau,\delta}_{DP}$ satisfies $(\tau,\delta)-DP$ and that $Pr[\mathcal{M}_s(D_{-x})\in O]=\sum_{Z\subseteq D_{-x}} Pr[f_{s}(D_{-x})=Z]Pr[A(Z)\in O]$. The second equality is a rearrangement of terms. The second inequality is due to the fact that $1-\pi_i+\pi_i e^{\tau}\leq e^{\epsilon_i}$ by Equation~\ref{eq:sam}.

ii) $Pr[{\mathcal{M}_{s}}(D_{-x})\in O] \leq e^{\epsilon_i} Pr[\mathcal{M}_{s} (D) \in O] + \pi_i \delta$\\
We take a similar approach as the first case.
\begin{flalign*}
    Pr[{\mathcal{M}_{s}}(D)\in O]&= \pi_i \sum_{Z \subseteq D_{-x}} Pr[f_{s}(D_{-x})=Z]Pr[A (Z_{+x}) \in O] &&\\ 
    &\;\;\;\;+(1- \pi_i ) \sum_{Z \subseteq D_{-x}} Pr[f_{s}(D_{-x})=Z] Pr[A (Z) \in O] &&\\
    & \geq \pi_i \sum_{Z \subseteq D_{-x}} Pr[f_{s}(D_{-x})=Z]e^{-\tau} ( Pr[A (Z) \in O] - \delta )&&\\
    &\;\;\;\;+(1- \pi_i ) Pr[\mathcal{M}_{s} (D_{-x}) \in O]&&\\
    & = ( \pi_i e^{-\tau} +1 - \pi_i ) (Pr[\mathcal{M}_{s} (D_{-x}) \in O]&&\\
    &\;\;\;\;-\delta {\frac{\pi_i e^{-\tau}}{ \pi_i e^{-\tau} +1 - \pi_i}} \sum_{Z \subseteq D_{-x}} Pr[f_{s}(D_{-x}) =Z] )&&\\
    & \geq ( \pi_i e^{-\tau} +1 - \pi_i ) (Pr[\mathcal{M}_{s} (D_{-x}) \in O] - \pi_i \delta ) &&\\
    & \geq e^{- \epsilon_i} (Pr[\mathcal{M}_{s} (D_{-x} ) \in O] - \pi_i \delta )&&
\end{flalign*}

Here, the first inequality is due to the $(\tau,\delta)$-DP of algorithm $A^{\tau,\delta}_{DP}$. The second equality is a rearrangement of terms. The second inequality is due to the fact that the last term is less than or equal to $\delta$ ($\because \frac{\pi_i e^{-\tau}}{ \pi_i e^{-\tau} +1 - \pi_i}\leq \pi_i$ and $\sum_{D_{-x}} Pr[f_{s}(D_{-x}) =Z] \leq 1$). The third inequality is by $\pi_i e^{-\tau} + 1 - \pi_i \geq e^{-\epsilon_i}$.

Therefore, {\em PDP-SGD-sam} mechanism $\mathcal{M}_s$ satisfies the following inequality:
\begin{align*}
    & Pr[{\mathcal{M}_{s}}(D)\in O] \leq e^{\epsilon_i} Pr[\mathcal{M}_{s} (D') \in O] + \pi_i \delta.
\end{align*}
\end{proof}


\subsection{Proof of Theorem \ref{thm:cmp}}
\label{appendix:compositiontheoremproof}

Continuing from Section~\ref{sec:prove}, we prove the following theorem.

\compositionthm*

\begin{proof}
We can follow a similar method with ~\citet{dwork2013algorithmic}, which proved the composition theorem for $(\epsilon,\delta)$-DP. The difference is that this proof is the extended version of the method which can be fitted into $(\Phi, \Delta)$-PDP setting, so privacy parameters can be different for each function and each user.

For all $C_1 \subseteq \mathcal{C}_1$ and user $u_i \in \mathcal{U}$, let $\mu_i$ be a measure on $\mathcal{C}_1$ such that 

$\mu_i (C_1) = |Pr[T_1 (D) \in C_1] - e^{\epsilon_{1i}} P[T_1 (D') \in C_1]|$ 

Because $T_1$ is $(\Phi_1, \Delta_1)$-PDP function, $\mu_i (C_1) \leq \delta_{1i}$.

Then for all $s_1 \in \mathcal{C}_1$, $Pr[T_1(D) \in ds_1]\leq e^{\epsilon_{1i}} Pr[T_1 (D') \in ds_1] +\mu_i(ds_1)$ 

where $ds_1$ be a infinitesimal change of $s_1$.
\begin{flalign*}
    Pr[(T_2 (D,T_1 (D)) , &T_1 (D) ) \in S]\leq \int_{S_1} Pr[(T_2 (D,s_1 ) , s_1) \in S] Pr[T_1 (D) \in ds_1 ]&& \\
    &\leq \int_{S_1} \min(e^{\epsilon_{2i}}Pr[(T_2 (D',s_1 ) , s_1) \in S] + \delta_{2i} , 1)\cdot Pr[T_1 (D) \in ds_1 ] &&\\
    &\;\;\;\;(\because Pr[(T_2(D,s_1),s_1) \in S] \leq 1)&&\\
    &\leq\int_{S_1} (\min(e^{\epsilon_{2i}}Pr[(T_2 (D',s_1 ) , s_1) \in S], 1) + \delta_{2i})\cdot Pr[T_1 (D) \in ds_1]&& \\
    &\leq \int_{S_1} \min(e^{\epsilon_{2i}}Pr[(T_2 (D',s_1 ) , s_1) \in S], 1)\cdot Pr[T_1 (D) \in ds_1]+\delta_{2i}&&\\
    &\leq\int_{S_1} \min(e^{\epsilon_{2i}}Pr[(T_2 (D',s_1 ) , s_1) \in S], 1)&&\\
    &\cdot (e^{\epsilon_{1i}}Pr[T_1 (D') \in ds_1] + \mu_i(ds_1))+\delta_{2i} &&\\
    &\leq e^{\epsilon_{1i}+\epsilon_{2i}}\int_{S_1}Pr[(T_2 (D',s_1 ) , s_1) \in S] Pr[T_1 (D') \in ds_1] + \mu_i(S_1 ) + \delta_{2i}&& \\
    &\leq e^{\epsilon_{1i}+\epsilon{2i}} Pr[(T_2 (D',T_1 (D')) , T_1 (D') ) \in S] + \delta_{1i} + \delta_{2i}&&
\end{flalign*}
\end{proof}

\subsection{Proof of Theorem \ref{thm:post}}
\label{appendix:postptheoremproof}

Continuing from Section~\ref{sec:prove}, we prove the following theorem.

\postpthm*

\begin{proof}
As mentioned in~\citet{dwork2013algorithmic}, if we prove for the case where $f$ is a deterministic function, we can apply it to the randomized mapping case because any randomized mapping can be expressed with the composition of convex combination of deterministic functions.

Let $f:X \rightarrow Y$ and for all $O \subseteq Y$, let $Z = \{ x \in X | f(x) \in O \}$. Then for each user $u_i \in \mathcal{U}$ and its associated data record $x$, and all datasets $D, D' \in \mathcal{D}$ with $D \stackrel x\sim D'$, 
\begin{equation*}
    \begin{aligned}
        Pr[f(\mathcal{M}(D)) \in O]&= Pr[\mathcal{M}(D) \in Z]\\
        &\leq e^{\epsilon_i} Pr[\mathcal{M} (D') \in Z] + \delta_i\\
        &= e^{\epsilon_i}Pr[f(\mathcal{M} (D')) \in O] + \delta_i.
    \end{aligned}
\end{equation*}
\end{proof}

\subsection{Proof of Theorem \ref{thm:rdp}}
\label{appendix:individualrdptopdp}

Continuing from Section~\ref{sec:prove}, we prove the following theorem.

\rdpthm*

\begin{proof}
    If $\mathcal{M}$ satisfies $D_\alpha (\mathcal{M}(D)||\mathcal{M}(D^{-i})) \leq \rho_i$, it also satisfies
\begin{align*}
    Pr[\mathcal{M}(D)\in S] \leq \{e^{\rho_i}Pr[\mathcal{M}(D^{-i})\in S]\}^{1-1/\alpha}
\end{align*}

i) $e^{\rho_i}Pr[\mathcal{M}(D^{-i})\in S]>\delta_i^{\alpha/(\alpha-1)}$
\begin{align*}
    Pr[\mathcal{M}(D)\in S] &\leq \{e^{\rho_i}Pr[\mathcal{M}(D^{-i})\in S]\}^{1-1/\alpha}\\
    &\leq e^{\rho_i}Pr[\mathcal{M}(D^{-i})\in S]\cdot\delta_i^{-1/(\alpha-1)}\\
    &=exp\left ( \rho_i + \frac{log1/\delta_i}{\alpha-1} \right )Pr[\mathcal{M}(D^{-i})\in S]
\end{align*}

ii) $e^{\rho_i}Pr[\mathcal{M}(D^{-i})\in S] \leq \delta_i^{\alpha/(\alpha-1)}$
\begin{align*}
    Pr[\mathcal{M}(D)\in S] \leq \{e^{\rho_i}Pr[\mathcal{M}(D^{-i})\in S]\}^{1-1/\alpha}
    \leq \delta_i
\end{align*}

Therefore,
\begin{align*}
    Pr[\mathcal{M}(D)\in S] &\leq max(e^{\epsilon_i}Pr[\mathcal{M}(D^{-i})\in S],\,\delta_i)\\
    & \leq e^{\epsilon_i}Pr[\mathcal{M}(D^{-i})\in S]+\delta_i
\end{align*}
where $\epsilon_i=\rho_i + \frac{log1/\delta_i}{\alpha-1}$.
\end{proof}

\end{document}